\documentclass[letterpaper]{article} %
\usepackage{aaai25}  %
\usepackage{times}  %
\usepackage{helvet}  %
\usepackage{courier}  %
\usepackage[hyphens]{url}  %
\usepackage{graphicx} %
\usepackage{natbib}  %
\usepackage{caption} %
\usepackage{amsmath}
\usepackage{amsfonts}
\usepackage{todonotes}
\usepackage{booktabs}
\usepackage{subcaption}
\usepackage{amsthm}
\usepackage{adjustbox}
\usepackage{multirow}
\usepackage{nicematrix}
\usepackage{scalefnt}

\usepackage{tikz}
\usetikzlibrary{positioning, arrows, arrows.meta, trees, shapes.geometric}

\usepackage[ruled,linesnumbered,noend,noline]{algorithm2e}
\DontPrintSemicolon
\SetKwComment{Comment}{$\triangleright$\ }{}
\SetKwInOut{Input}{Input}
\SetKwInOut{Output}{Output}
\SetKwInput{KwCondONE}{Cond1}

\newcommand{\pbcount}{\ensuremath{\mathsf{PBCount}}}

\newcommand{\pbcounttwo}{\ensuremath{\mathsf{PBCount2}}}
\newcommand{\pbcountcg}{{\pbcounttwo}}

\newcommand{\greedymergetext}{Least Occurrence Weighted Min Degree}
\newcommand{\greedymergeit}{\textit{\greedymergetext}}
\newcommand{\greedymergeshort}{\ensuremath{\mathsf{LOW}\text{-}\mathsf{MD}}}

\newcommand{\apply}{\ensuremath{\mathsf{Apply}}}

\newcommand{\func}[1]{\ensuremath{\mathsf{Func}}(#1)}
\newcommand{\dpmc}{\ensuremath{\mathsf{DPMC}}}
\newcommand{\gpmc}{\ensuremath{\mathsf{GPMC}}}
\newcommand{\exactmc}{\ensuremath{\mathsf{ExactMC}}}
\newcommand{\approxmcpb}{\ensuremath{\mathsf{ApproxMCPB}}}
\newcommand{\dfour}{\ensuremath{\mathsf{D4}}}
\newcommand{\addmc}{\ensuremath{\mathsf{ADDMC}}}
\newcommand{\pblib}{\ensuremath{\mathsf{PBLib}}}

\newcommand{\var}[1]{\ensuremath{\mathsf{Var}(#1)}}

\newcommand{\addproj}{\ensuremath{\Sigma\text{-}\mathsf{projection}}}
\newcommand{\orproj}{\ensuremath{\exists\text{-}\mathsf{projection}}}

\newcommand{\child}[1]{\ensuremath{\mathsf{Ch}(#1)}}

\newtheorem{definition}{Definition}
\newtheorem{lemma}{Lemma}
\newtheorem{theorem}{Theorem}

\newtheorem{innercustomthm}{Theorem}
\newenvironment{mythm}[1]
  {\renewcommand\theinnercustomthm{#1}\innercustomthm}
  {\endinnercustomthm}

\newcommand\ScaleExists[1]{\vcenter{\hbox{\scalefont{#1}$\exists$}}}

\DeclareMathOperator*\bigexists{%
  \vphantom\sum
  \mathchoice{\ScaleExists{2}}{\ScaleExists{1.4}}{\ScaleExists{1}}{\ScaleExists{0.75}}
}

\title{Towards Projected and Incremental Pseudo-Boolean Model Counting}
\author {
    Suwei Yang\textsuperscript{\rm 1,\rm 2,\rm 4},
    Kuldeep S. Meel\textsuperscript{\rm 3,\rm 5}
}
\affiliations {
    \textsuperscript{\rm 1}Grabtaxi Holdings\\
    \textsuperscript{\rm 2}Grab-NUS AI Lab\\
    \textsuperscript{\rm 3}Georgia Institute of Technology\\
    \textsuperscript{\rm 4}National University of Singapore\\
    \textsuperscript{\rm 5}University of Toronto
}

\begin{document}

\maketitle

\begin{abstract}
Model counting is a fundamental task that involves determining the number of satisfying assignments to a logical formula, typically in conjunctive normal form (CNF). While CNF model counting has received extensive attention over recent decades, interest in Pseudo-Boolean (PB) model counting is just emerging partly due to the greater flexibility of PB formulas. As such, we observed feature gaps in existing PB counters such as a lack of support for projected and incremental settings, which could hinder adoption.

In this work, our main contribution is the introduction of the PB model counter {\pbcounttwo}, the first exact PB model counter with support for projected and incremental model counting. Our counter, {\pbcounttwo}, uses our {\greedymergetext} ({\greedymergeshort}) computation ordering heuristic to support projected model counting and a cache mechanism to enable incremental model counting. In our evaluations, {\pbcounttwo} completed at least 1.40$\times$ the number of benchmarks of competing methods for projected model counting and at least 1.18$\times$ of competing methods in incremental model counting.
\end{abstract}

\section{Introduction} \label{sec:introduction}

Pseudo-Boolean (PB) model counting involves determining the number of satisfying assignments of a given pseudo-boolean formula, which differs from typical model counting tasks where the input is a Boolean formula in conjunctive normal form (CNF). In recent years, there has been increasing interest from the community in PB model counting and PB formulas in general, in part due to their succinctness and flexibility over CNF formulas~\cite{S05}. In particular, an arbitrary CNF clause can always be converted to a single PB constraint, but the converse is not true~\cite{LMMW18}. The emerging interests take the form of new PB model counting tools such as {\approxmcpb} and {\pbcount}~\cite{YM21,YM24} as well as the emergence of applications that include knapsack problems, neural network verification, and budgeted sensor placement~\cite{P05,YM21,LSM23}. 

In contrast, CNF model counting and solving (SAT solving) are well-established fields, benefiting from many decades of research advancements that have led to numerous feature-rich counters and solvers. The availability of tools for CNF model counting and solving has led to an ever-increasing number of applications across a wide range of domains such as software design, explainable machine learning, planning, and probabilistic reasoning~\cite{BDP03,NSMIM19,J19,FMM20}. In turn, the wide range of applications drives demand for better tools, from both performance and feature perspectives. Subsequently, better tools would lead to more adoption and applications. Eventually, a positive self-reinforcing loop is formed within the CNF model counting and solving community. On the other hand, we have yet to observe such a positive cycle in the PB model counting community, in part due to a lack of features such as exact projected counting and incremental setting support in existing PB model counting tools.

In this work, our main contribution is the PB model counter, {\pbcounttwo}, which supports projected and incremental model counting features. We focus on these two aforementioned features due to their importance to the CNF model counting and SAT solving community. Projected model counting involves determining the number of partial assignments over a given set of variables, known as the projection set, that can be extended to a satisfying assignment over all variables of the formula. Projected model counting has been applied in areas such as planning, verification, and reliability estimation~\cite{ACMS15,KMC13, DMPV17}. Similarly, incremental SAT solving has led to applications such as combinatorial testing, circuit design, and solving string constraints~\cite{YKACOB15,YZLCH17,LGDK23}. In this work, we focus on the incremental setting whereby a given PB formula undergoes incremental modifications, more specifically addition and removal of PB constraints, while having its model count computed along each modification step. In order to support projected PB model counting, we introduce a new computation ordering heuristic that is compatible with the ordering requirements for projected model counting. Additionally, we introduce a new caching mechanism that is at the core of incremental counting. {\pbcounttwo} is to the best of our knowledge both the first exact projected PB counter as well as the first incremental model counter. 

We performed extensive evaluations on benchmark instances inspired by various application settings, such as sensor placement, multi-dimension knapsack, and combinatorial auction benchmarks~\cite{GL80,BN07,LSM23}. Our evaluations highlighted the efficacy of {\pbcountcg} at projected PB model counting and incremental model counting compared to the baseline techniques. In particular, {\pbcountcg} is able to successfully count 1957 projected model counting instances while the state-of-the-art CNF-based projected model counter could only count 1398 instances. 
Incremental benchmarks involve multiple modifications of the initially provided PB formula, with model counts computed after each modification step. Our experiments showed that {\pbcountcg} is able to complete 1618 instances for incremental benchmarks involving 5 counts with 1 count for the initial PB formula and 4 counts for modification steps. In comparison, the state-of-the-art PB model counter, {\pbcount}, completed only 1371 instances, demonstrating a significant performance advantage of {\pbcountcg} at incremental counting. 
Moreover, {\pbcounttwo} also demonstrates superior performance at incremental settings compared to the state-of-the-art CNF model counters {\dfour} and {\gpmc}, both of which completed less than 1000 incremental benchmarks. 
Given that it is the early days of PB model counting, we hope the new capabilities introduced in this work will attract more interest and applications of PB model counting, leading to a positive self-reinforcing cycle within the PB model counting community.

\section{Notations and Preliminaries} \label{sec:background}

\paragraph{Pseudo-Boolean Formula}

A PB formula $F$ consists of one or more PB constraints, each of which is either equality or inequality. A PB constraint takes the form $\sum_{i=1}^{n} a_i x_i \square k$ where $x_1, ..., x_n$ are Boolean literals, $a_1, ..., a_n$, and $k$ are integers, and $\square$ is one of $\{\geq, =, \leq\}$. We refer to $a_1, ..., a_n$ as term coefficients in the PB constraint, where each term is of the form $a_i x_i$. $F$ is satisfiable if there exists an assignment $\tau$ of all variables of $F$ such that all its PB constraints evaluate to \textit{true}. PB model counting refers to the computation of the number of satisfying assignments of $F$.

\paragraph{Projected Model Counting}
Let $F$ be a formula and \var{F} be the set of variables of $F$. Let $X$ and $Y$ be disjoint subsets of \var{F} such that $X \cap Y = \emptyset$  and $X \cup Y = \var{F}$. $(F,X,Y)$ is a projected model counting instance. The projected model count is $\sum_{\beta \in 2^X} ( \max_{\alpha \in 2^Y} [F](\alpha \cup \beta))$~\cite{DPM21}. Where $[F](\alpha \cup \beta)$ is the evaluation of $F$ with the variable assignment $\alpha \cup \beta$, and returns 1 if $F$ is satisfied and 0 otherwise. In other words, the projected model count of $F$ on $X$ is the number of assignments of all variables in $X$ such that there exists an assignment of variables in $Y$ that makes $F$ evaluate to \textit{true}~\cite{ACMS15}. Non-projected model count is a special case of projected model counting where all variables are in the projection set, i.e. $X = \var{F}$, and $Y = \emptyset$. With reference to Figure~\ref{fig:add-constraint}, suppose the PB formula has only the single PB constraint $2 x_1 + x_2 + x_3 \geq 2$, then the 5 satisfying assignments are $(x_1, x_2, x_3), (x_1, x_2, \bar{x_3}), (x_1, \bar{x_2}, x_3), (x_1, \bar{x_2}, \bar{x_3}),$ and $(\bar{x_1}, x_2, x_3)$. If the projection set is ${x_1}$ then the corresponding projected model count is 2, as both partial assignments involving $x_1$ can be extended to a satisfying assignment.

\paragraph{Projections}
Let $f : 2^{X} \rightarrow \mathbb{R}$ be a function defined over a set of boolean variables $X$. The {\addproj} of $f$ with respect to a variable $x \in X$ for all $\sigma$ is given by $f(\sigma \land x) + f (\sigma \land \bar{x})$ where $\sigma$ refers to assignments of all variables excluding $x$. Similarly, the {\orproj} with respect to $x$ is given by $\max \left( f(\sigma \land x), f (\sigma \land \bar{x}) \right)$. In the case that $f$ maps to 1 or 0, i.e. $f : 2^{X} \rightarrow \{ 0,1 \}$, then {\orproj} can be written as $f(\sigma \land x) \vee f (\sigma \land \bar{x})$. In this work, we use the two types of projections to compute the projected model count in Algorithm~\ref{alg:pbcount-greedy-mc} which we detail in Section~\ref{subsec:projected-mc}.

\paragraph{Algebraic Decision Diagram}

An algebraic decision diagram (ADD)~\cite{BFGHMPF93} is a directed acyclic graph (DAG) representation of a function $f: 2^{\var{F}} \rightarrow S$ where \var{F} is the set of Boolean variables that $f$ is defined over, and $S$ is known as the carrier set. We denote the function represented by an ADD $\psi$ as $\func{\psi}$. The internal nodes of ADD represent decisions on variables $x \in \var{F}$ and the leaf nodes represent $s \in S$. The order of variables represented by decision nodes, from root to leaf of ADD, is known as the variable ordering. In this work, $S$ is a set of integers. As example, an ADD representing $2 x_1 + x_2 + x_3$ is shown in Figure~\ref{fig:add-expression}, and $2 x_1 + x_2 + x_3 \geq 2$ in Figure~\ref{fig:add-constraint}. The dashed arrows of internal nodes represent when the corresponding variables are set to \textit{false}, and solid arrows represent when they are set to \textit{true}.

\begin{figure}
  \centering
  \begin{adjustbox}{width=0.55\columnwidth}
  \begin{tikzpicture}[
  roundnode/.style={circle, draw=black!60, very thick, minimum size=7mm},
  ]
  \node[roundnode](x1) at (0, 0){$x_1$};
  \node[roundnode](x2left) at (-1, -1.0){$x_2$};
  \node[roundnode](x2right) at (1, -1.0){$x_2$};
  \node[roundnode](x3ll) at (-2, -2.20){$x_3$};
  \node[roundnode](x3lr) at (-0.75, -2.20){$x_3$};
  \node[roundnode](x3rl) at (0.75, -2.20){$x_3$};
  \node[roundnode](x3rr) at (2, -2.20){$x_3$};

  \node[roundnode](leaf0) at (-2.5, -3.4){$0$};
  \node[roundnode](leaf1) at (-1.25, -3.4){$1$};
  \node[roundnode](leaf2) at (0, -3.4){$2$};
  \node[roundnode](leaf3) at (1.25, -3.4){$3$};
  \node[roundnode](leaf4) at (2.5, -3.4){$4$};

  \draw[dashed,->] (x1) -- (x2left);
  \draw[->] (x1) -- (x2right);
  \draw[dashed,->] (x2left) -- (x3ll);
  \draw[->] (x2left) -- (x3lr);
  \draw[dashed,->] (x2right) -- (x3rl);
  \draw[->] (x2right) -- (x3rr);

  \draw[dashed,->] (x3ll) -- (leaf0);
  \draw[->] (x3ll) -- (leaf1);
  \draw[dashed,->] (x3lr) -- (leaf1);
  \draw[->] (x3lr) -- (leaf2);
  \draw[dashed,->] (x3rl) -- (leaf2);
  \draw[->] (x3rl) -- (leaf3);
  \draw[dashed,->] (x3rr) -- (leaf3);
  \draw[->] (x3rr) -- (leaf4);

  \end{tikzpicture}
  \end{adjustbox}
  \caption{An ADD representing $2 x_1 + x_2 + x_3$ with ordering $x_1,x_2,x_3$}
  \label{fig:add-expression}
\end{figure}
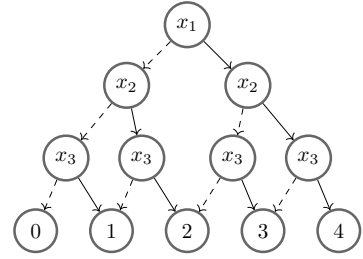

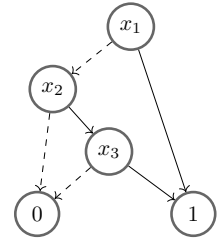
\begin{figure}
  \centering
  \begin{adjustbox}{width=0.32\columnwidth}
  \begin{tikzpicture}[
      roundnode/.style={circle, draw=black!60, very thick, minimum size=7mm},
  ]
  \node[roundnode](x1) at (0, 0){$x_1$};
  \node[roundnode](x2left) at (-1.25, -1.0){$x_2$};
  \node[roundnode](x3left) at (-0.35, -2.0){$x_3$};
  \node[roundnode](leaf0) at (-1.5, -3.0){$0$};
  \node[roundnode](leaf1) at (1, -3.0){$1$};

  \draw[dashed,->] (x1) -- (x2left);
  \draw[->] (x1) -- (leaf1);
  \draw[dashed,->] (x2left) -- (leaf0);
  \draw[->] (x2left) -- (x3left);
  \draw[dashed,->] (x3left) -- (leaf0);
  \draw[->] (x3left) -- (leaf1);

  \end{tikzpicture}
  \end{adjustbox}
  \caption{An ADD representing $2 x_1 + x_2 + x_3 \geq 2$ with ordering $x_1,x_2,x_3$}
  \label{fig:add-constraint}
\end{figure}

In this work, we make use of the {\apply} operation on ADDs~\cite{B86,BFGHMPF93}. The {\apply} operation takes as input a binary operator $\bowtie$, two ADDs $\psi_1, \psi_2$, and outputs an ADD $\psi_3$ such that the $\func{\psi_3} = \func{\psi_1} \bowtie \func{\psi_2}$. It is worth noting that the input ADDs $\psi_1, \psi_2$ are required to have the same ADD variable orderings i.e. the ordering of variables being represented by internal nodes of ADDs from root to leaf nodes. 
In this work, we use the term \textit{merge} to refer to the usage of {\apply} with $\times$ operator on two ADDs. Referring to the explanation of {\apply}, we say $\psi_3$ represents PB constraints $\{c_1$, $c_2\}$ if $\psi_1$ represents $\{c_1\}$ and $\psi_2$ represents $\{c_2\}$.

\paragraph{Project-Join Tree} \label{app:def:pj-tree}
Let $F$ be a formula. A project-join tree~\cite{DPV20b} of $F$ is a tuple $\mathcal{T}=(T,r,\gamma, \pi)$ where $T$ is a tree with nodes $\mathcal{V}(T)$ and rooted at node $r$. $\gamma$ is a bijection from leaves $\mathcal{L}(T)$ of $T$ to constraints of $F$ and $\pi$ is a labelling function $\mathcal{V}(T) \setminus \mathcal{L}(T) \rightarrow 2^{\var{F}}$ on internal nodes $\mathcal{V}(T) \setminus \mathcal{L}(T)$ of $T$. Additionally, $\mathcal{T}$ satisfies the following:
\begin{enumerate}
    \item The set $\{\pi(n) : n \in \mathcal{V}(T) \setminus \mathcal{L}(T)\}$ partitions \var{F}
    \item For an arbitrary internal node $n$. Let $x$ be a variable in $\pi{n}$, and $c$ be a PB constraint of $F$. If $x \in \var{c}$, then leaf node $\gamma^{-1} (c)$ is a descendant of $n$.
\end{enumerate}

\paragraph{X,Y-Graded Project-Join Tree} \label{app:def:xy-pj-tree}
Let $F$ be a formula with project join tree $\mathcal{T}$, and variable sets $X,Y$ being partitions of \var{F}. The project join tree $\mathcal{T}$ is an $X,Y$ graded project-join tree~\cite{DPM21} if there exist grades $\mathcal{I}_X$ and $\mathcal{I}_Y$ such that:
\begin{enumerate}
    \item The set $\{\mathcal{I}_X , \mathcal{I}_Y\}$ is a partition of internal nodes of $\mathcal{T}$
    \item If node $n_X \in \mathcal{I}_X$, $\pi(n_X) \subseteq X$
    \item If node $n_Y \in \mathcal{I}_Y$, $\pi(n_Y) \subseteq Y$
    \item If $n_X \in \mathcal{I}_X$ and $n_Y \in \mathcal{I}_Y$, $n_X$ is not a descendant of $n_Y$ in tree $T$ rooted at $r$.
\end{enumerate}

\paragraph{Projected Model Counting with Project Join Trees} \label{subsec:pmc-proj-join-tree}
Over a series of works~\cite{DPV20a,DPV20b,DPM21} \citeauthor{DPM21} established a CNF model counting framework with project join trees and extended the framework to projected CNF model counting by employing a specific type of project join tree, namely an $X,Y$-graded project join tree. \citeauthor{DPM21} showed that performing {\orproj} at $\mathcal{I}_{Y}$ nodes and {\addproj} at $\mathcal{I}_{X}$ according to $\mathcal{T}$ produces the correct projected CNF model count~\cite{DPM21}. We show that our projected PB model counting approach shares similarities with the project join tree framework in Section~\ref{subsec:projected-mc}, providing the intuition for algorithm correctness.

\section{Related Work} \label{sec:related-works}

\paragraph{Search-Based Model Counters}
We can classify the numerous existing model counters into two main categories based on their underlying techniques -- search-based model counters and decision diagram-based model counters. Search-based model counters iteratively assign values to variables, implicitly exploring a search tree of variable assignments while keeping count throughout the process. The counters typically employ component caching, so that counts for each branch of the search tree are only computed once. Notable search-based model counters include {\gpmc}, $\mathsf{Ganak}$, and \textsf{Sharpsat-TD}~\cite{SHS17,SRSM19,KJ21}.

\paragraph{Decision Diagram-Based Model Counters}
The use of different types of decision diagrams, which are directed acyclic graph (DAG) representations of the assignment space of a given formula, is common among model counters, for both PB and CNF formulas. Recent decision diagram-based model counters include {\dfour}, {\exactmc}, {\addmc}, and {\dpmc} for CNF model counting and {\pbcount} for PB model counting~\cite{LM17,DPV20a,DPV20b,LMY21,YM24}. {\dfour} and {\exactmc} build their respective decision diagrams in a top-down manner, forming a single diagram representing the assignment space. In contrast, {\addmc}, {\dpmc}, and {\pbcount} build decision diagrams in a bottom-up manner, starting from individual decision diagrams of clauses or constraints and subsequently combining the decision diagrams to represent the assignment space. Our counter, which we term {\pbcountcg}, follows a similar methodology and combines individual decision diagrams in a particular order to support projected model counting while also caching the intermediate diagrams during the process.

\paragraph{Pseudo-Boolean Model Counting}
There has been recent interest in PB model counting. Notable tools developed for PB model counting include the approximate PB model counter {\approxmcpb}~\cite{YM21} and the recently introduced exact PB model counter {\pbcount}~\cite{YM24}. {\approxmcpb} uses hash functions to evenly split the space of satisfying assignments and enumerates a partition of satisfying assignments to obtain an estimated count. On the other hand, {\pbcount} follows the methodology introduced in {\addmc}~\cite{DPV20a} to exactly count the number of satisfying assignments of the input PB formula. However, {\pbcount} does not support projected model counting or incremental counting features. 
In this work, we introduce the projected model counting and incremental counting features to the community via our counter {\pbcountcg}.

\paragraph{Incremental SAT Solving}
There have been numerous works in existing literature regarding incremental settings, more specifically for incremental satisfiability (SAT) solving~\cite{NR12,NRS14,FBS19,N22}. Incremental SAT solving involves finding satisfying assignments to a user-provided Boolean formula, typically in CNF, under certain assumptions that users can incrementally specify. Specifically, users can add and remove CNF clauses to the initially specified CNF formula. Incremental SAT solving is also useful for cases where users have to solve a large number of similar CNF formulas and has led to a wide range of applications such as combinatorial testing, circuit design, and solving string constraints~\cite{YKACOB15,YZLCH17,LGDK23}. In this work, we introduce the concept of incrementality to model counting tasks, by making use of a caching mechanism which we detail in Section~\ref{subsec:inc-counting}.

\section{Approach} \label{sec:approach}

\begin{figure*}[htb]
  \centering
  \begin{adjustbox}{width=0.75\textwidth}
  \begin{tikzpicture}[
  node distance=2cm,
  diamondnode/.style={diamond, minimum width=1.5cm, minimum height=0.5cm, text centered, draw=black, fill=green!30},
  rectnode/.style={rectangle, rounded corners, minimum width=1.5cm, minimum height=0.5cm,text centered, draw=black, fill=red!30},
  rectnodenofill/.style={rectangle, rounded corners, minimum width=1.5cm, minimum height=0.5cm,text centered, draw=black},
  ]
  \node (preprocess-compile) [rectnodenofill, text width=2.4cm] at (0,0) {Preprocess \& compile into \\ individual ADDs \\ (lines~\ref{alg-line:pbcount-greedy-mc:preprocess}-\ref{alg-line:pbcount-greedy-mc:compile-add-end})};
  \node (or-project) [rectnode, text width=2.4cm] at (3, 0) {Merge ADDs \& {\orproj} of non-projection set variables \\ (lines~\ref{alg-line:pbcount-greedy-mc:non-projection-start}-\ref{alg-line:pbcount-greedy-mc:cache1})};
  \node (add-project) [rectnode, text width=2.3cm] at (6, 0) {Merge ADDs \& {\addproj} of projection set variables \\ (lines~\ref{alg-line:pbcount-greedy-mc:projection-start}-\ref{alg-line:pbcount-greedy-mc:cache2})};

  \node (pbformula) at (-3, 0) {PB formula $F$};
  \node (output-count) at (9, 0) {Model Count};

  \node (approach name) at (3.1, -1.5) {Projected model counting flow for {\pbcountcg}};

  \draw [->] (pbformula) -- (preprocess-compile);
  \draw [->] (preprocess-compile) -- (or-project);
  \draw [->] (or-project) -- (add-project);
  \draw [->] (add-project) -- (output-count);

  \draw[dashed] (-1.5, 1.35) rectangle (7.5,-1.7);

  \end{tikzpicture}
  \end{adjustbox}
  \caption{Overall flow of our projected model counter {\pbcountcg}. Shaded boxes indicate our contributions and white box indicates techniques adapted from existing works. Line numbers correspond to lines in Algorithm~\ref{alg:pbcount-greedy-mc}.}
  \label{fig:projectmc-flow}
\end{figure*}
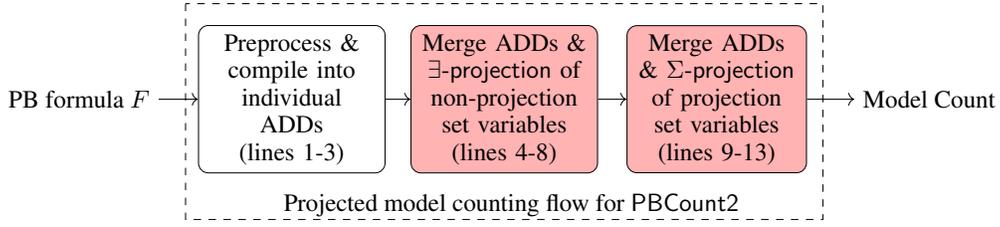

In this section, we detail our {\greedymergetext} ({\greedymergeshort}) computation ordering heuristic as well as the implementations of {\pbcountcg} to support a) projected model counting and b) incremental model counting.

\subsection{{\greedymergetext}} 

The general methodology of performing PB model counting tasks with ADDs involves representing each constraint with its individual ADD and subsequently merging the ADDs and projecting away variables to produce the final model count. As implemented in existing work~\cite{YM24}, the ordering of merging ADDs and projecting away variables is determined by heuristics. However, the existing heuristics did not support projected model counting. To this end, we introduce a new ordering heuristic in {\pbcounttwo}, which we term the {\greedymergeit} heuristic ({\greedymergeshort} in short). 

\begin{algorithm}
  \caption{{\pbcountcg} model counting with {\greedymergeshort} heuristic}\label{alg:pbcount-greedy-mc}
  
  \Input{PB formula $F$, Projection set $X$, Non-projection set $Y$}
  \Output{Model count}
  
  $F \gets \mathsf{preprocess}(F)$\; \label{alg-line:pbcount-greedy-mc:preprocess}
  \For{\upshape all constraints $c$ of $F$}{
    Compile ADD for $c$\; \label{alg-line:pbcount-greedy-mc:compile-add-end}
  }
  \While{$Y \neq \emptyset$}{ \label{alg-line:pbcount-greedy-mc:non-projection-start}
    $\psi \gets \mathsf{ADD(1)}$ \Comment*[r]{$\mathsf{ADD(1)}$ returns ADD representing 1}
    $y \gets \mathsf{popNextVar}(Y)$\; \label{alg-line:pbcount-greedy-mc:greedy1}
    \lForEach{\upshape ADD $\varphi$ containing variable $y$} {
      $\psi \gets \psi \times \varphi$ \label{alg-line:pbcount-greedy-mc:nps-merge}
    }
    $\psi \gets \psi[y \mapsto 1] \vee \psi[y \mapsto 0]$ \label{alg-line:pbcount-greedy-mc:cache1} \Comment*[r]{{\orproj}}
  }
  \While{$X \neq \emptyset$}{ \label{alg-line:pbcount-greedy-mc:projection-start}
    $\psi \gets \mathsf{ADD(1)}$\;
    $x \gets \mathsf{popNextVar}(X)$\;\label{alg-line:pbcount-greedy-mc:greedy2}
    \lForEach{\upshape ADD $\varphi$ containing variable $x$} {
      $\psi \gets \psi \times \varphi$ \label{alg-line:pbcount-greedy-mc:ps-merge}
    }
    $\psi \gets \psi[x \mapsto 1] + \psi[x \mapsto 0]$ \label{alg-line:pbcount-greedy-mc:cache2} \Comment*[r]{{\addproj}}
  }
  $\psi \gets \mathsf{ADD(1)}$\;
  \lForEach{\upshape remaining intermediate ADD $\varphi$} {
    $\psi \gets \psi \times \varphi$ \label{alg-line:pbcount-greedy-mc:final-merge}
  }
  \Return $\psi\mathsf{.value}$\;
\end{algorithm}

Our {\greedymergeit} heuristic is as follows. Let $G$ be an undirected bipartite graph, where the vertices either represent variables in PB formula $F$ or a single PB constraint of $F$. A variable vertex $v_x$ is connected to a constraint vertex $v_c$ if the variable appears in that PB constraint. The {\greedymergeshort} heuristic entails picking the variable that has the corresponding variable vertex in $G$ with the minimum degree. The heuristic is equivalent to a weighted version of the min-degree heuristic on Gaifman graphs, where weights correspond to the number of PB constraints in which two variables appear together.

The intuition behind {\greedymergeshort} stems from the observation that the algorithmic complexity of merging two ADDs of size $m, n$ is on the order of $\mathcal{O}(mn)$. As such, we would like to reduce the size of operand ADDs as much as possible, especially when the overall model counting algorithm involves many such ADD merging operations. In the computation process, the size of an ADD reduces when a variable is projected away. To ensure correctness, a variable can only be projected away when all ADDs involving it have been merged~\cite{DPV20a}. Hence, we designed our {\greedymergeshort} heuristic to pick the least frequently occurring variable to project away, as it involves merging the fewest number of ADDs before projecting away the variable.

\subsection{Projected Model Counting} \label{subsec:projected-mc}

Recall that in projected model counting, there are two non-overlapping sets of variables $X, Y$ where $X$ is the projection set and $Y$ is the non-projection set. The key idea to support projected model counting is the different way variables in $X$ and $Y$ are projected away. For all variables $x \in X$, we project $x$ away from an ADD $\psi$ using $\psi \gets \psi[x \mapsto 1] + \psi[x \mapsto 0]$, also referred to as {\addproj}~\cite{DPM21}. In contrast, for all variables $y \in Y$, we project away $y$ from $\psi$ using $\psi \gets \psi[y \mapsto 1] \vee \psi[y \mapsto 0]$, also referred to as {\orproj}~\cite{DPM21}. In addition, all variables in $Y$ must be projected away before any variable in $X$ because the different projection operations are not commutative across variables in $X$ and $Y$~\cite{DPM21}. To this end, we introduce the {\greedymergeshort} ordering which is compatible with the projected model counting ordering requirements. The overall flow of {\pbcountcg} is shown in Figure~\ref{fig:projectmc-flow}, and the pseudocode is shown in Algorithm~\ref{alg:pbcount-greedy-mc}.

In Algorithm~\ref{alg:pbcount-greedy-mc}, we employ the same preprocessing (line~\ref{alg-line:pbcount-greedy-mc:preprocess}) and individual constraint compilation techniques (line~\ref{alg-line:pbcount-greedy-mc:compile-add-end}) as the original {\pbcount}. Next, we process each variable $y$ in non-projection set $Y$ by merging all ADDs containing $y$ and projecting away $y$ from the merged ADD (lines~\ref{alg-line:pbcount-greedy-mc:non-projection-start}-\ref{alg-line:pbcount-greedy-mc:cache1}). In lines~\ref{alg-line:pbcount-greedy-mc:projection-start}-\ref{alg-line:pbcount-greedy-mc:cache2}, we do the same for variables in projection set $X$. In each iteration of the merge and project process, we select a variable using our {\greedymergeshort} heuristic, indicated by $\mathsf{popNextVar}(\cdot)$ on lines~\ref{alg-line:pbcount-greedy-mc:greedy1} and~\ref{alg-line:pbcount-greedy-mc:greedy2}, and remove it from $X$ or $Y$ respectively. As discussed previously, the {\greedymergeshort} ordering heuristic entails picking the variable that has the least occurrence in the ADDs at that moment of the computation. 

\paragraph*{Algorithm Correctness}

The algorithm correctness of projected model counting of {\pbcounttwo} follows prior work on projected CNF model counting with ADDs~\cite{DPM21}. 
\citeauthor{DPM21} showed that for projected CNF model counting correctness, the computations should be performed according to an $X,Y$-graded project join tree. In particular, performing $\orproj$ at $\mathcal{I}_{Y}$ nodes and $\addproj$ at $\mathcal{I}_{X}$ nodes of an $X,Y$-graded project join tree $\mathcal{T}$ produces the correct projected model count. {\pbcounttwo}'s algorithm correctness for projected PB model counting comes from the fact that the computation in Algorithm~\ref{alg:pbcount-greedy-mc} with {\greedymergeshort} heuristic implicitly follows an $X,Y$-graded project join tree, and therefore produces the correct count.

\begin{theorem} \label{theorem:count-correctness}
  Let $F$ be a formula defined over $X \cup Y$ such that $X$ is the projection set, and $Y$ is the set of variables not in projection set, then given an instance $(F,X,Y)$, Algorithm ~\ref{alg:pbcount-greedy-mc} returns $c$ such that $c = \sum_{\beta \in 2^X} (\max_{\alpha \in 2^Y} [F](\alpha \cup \beta))$
\end{theorem}
\begin{proof}
  The proof is deferred to the Appendix. 
\end{proof}

An $X,Y$-graded project join tree $\mathcal{T}$ has two sets of disjoint variables, in Algorithm~\ref{alg:pbcount-greedy-mc} this corresponds to variables in non-projection set $Y$ and projection set $X$. The initial individual ADDs produced at line~\ref{alg-line:pbcount-greedy-mc:compile-add-end} of Algorithm~\ref{alg:pbcount-greedy-mc} each corresponds to a leaf node in $\mathcal{T}$. Each of the intermediate ADDs during the merge and project iterations of variables in non-projection set $Y$ (lines~\ref{alg-line:pbcount-greedy-mc:non-projection-start} to~\ref{alg-line:pbcount-greedy-mc:cache1}) corresponds to an internal node of $\mathcal{T}$, in grade $\mathcal{I}_{Y}$. Similarly, each intermediate ADD during the merge and project process of projection set variables $X$ (lines~\ref{alg-line:pbcount-greedy-mc:projection-start} to~\ref{alg-line:pbcount-greedy-mc:cache2}) corresponds to an internal node of $\mathcal{T}$ in grade $\mathcal{I}_{X}$. Realize that no nodes of $\mathcal{T}$ in $\mathcal{I}_{X}$ are descendants of any node in $\mathcal{I}_{Y}$, satisfying the definition of $X,Y$-graded project join tree. In addition, we are performing {\orproj} at $\mathcal{I}_{Y}$ nodes and {\addproj} at $\mathcal{I}_{X}$ nodes. As such the computation process in Algorithm~\ref{alg:pbcount-greedy-mc} can be cast under the graded project join tree framework and would therefore follow the same proof as prior work~\cite{DPM21}.

\subsection{Incremental Counting} \label{subsec:inc-counting}

\begin{algorithm}
  \caption{Cache retrieval of {\pbcountcg}}\label{alg:inc-cache-retrieval}
  
  \Input{PB formula $F'$}
  \Output{Compute state - set of ADDs retrieved from cache}
  
  $C' \gets \mathsf{getConstraintSet}(F')$\;
  $\mathcal{A}, \mathcal{B} \gets \emptyset$ \Comment*[r]{Initialize 2 empty sets}
  \For{\upshape each ADD $\psi$ in cache} {
    \lIf{\upshape $\psi\mathsf{.constraints}$ $\subset C'$ and $\mathsf{CheckNoExtraVar}(\psi, C')$}{insert $\psi$ into $\mathcal{A}$} \label{alg-line:inc-cache-retrieval:formula-check}
  }
  \For{\upshape each ADD $\psi_\mathcal{A} \in \mathcal{A}$}{ \label{alg-line:inc-cache-retrieval:pairwise-check-start}
    $\mathsf{conflicts} \gets false$\; 
    \For{\upshape each ADD $\psi_\mathcal{B} \in \mathcal{B}$} {
      \If{\upshape there exists $x$ which is projected away in $\psi_\mathcal{A}$ and not $\psi_\mathcal{B}$, or vice versa} {
        $\mathsf{conflicts} \gets true$\;
      }
    }
    \lIf{\upshape not $\mathsf{conflicts}$}{insert $\psi_\mathcal{A}$ into $\mathcal{B}$} \label{alg-line:inc-cache-retrieval:pairwise-check-end}
  }
  \For{\upshape all $c \in C'$}{
    \lIf{\upshape $c$ not represented in $\mathcal{B}$}{insert constraint ADD of $c$ into $\mathcal{B}$}
  }
  \Return{$\mathcal{B}$}\;
\end{algorithm}

{\pbcountcg} supports incremental PB model counting via the caching of intermediate ADDs during the handling of model count queries. 
In particular, {\pbcountcg} supports the removal and addition of constraints in the PB formula. 
With reference to Algorithm~\ref{alg:pbcount-greedy-mc}, we cache the ADDs $\psi$ at lines~\ref{alg-line:pbcount-greedy-mc:cache1} and~\ref{alg-line:pbcount-greedy-mc:cache2} respectively. In order to store the compute state associated with an ADD, for cache retrieval purposes, we store 3 pieces of information: 1) the set of constraints in PB formula that the ADD represents 2) the projection set variables of ADD that have been projected away and 3) the non-projection set variables of ADD that have been projected away.

The cache retrieval mechanism is shown in Algorithm~\ref{alg:inc-cache-retrieval}. When given PB formula $F'$, modified from $F$ by adding or removing constraints, the core idea is to loop through the ADDs in the cache and retrieve the compatible ADDs, replacing the initial ADDs at line~\ref{alg-line:pbcount-greedy-mc:compile-add-end} of Algorithm~\ref{alg:pbcount-greedy-mc}. An ADD $\psi$ is compatible with $F'$ if the set of constraints that $\psi$ represents is a subset of the constraints of $F'$. In addition, the variables that have been projected out from $\psi$ must not appear in any constraint of $F'$ that $\psi$ does not represent, this is handled by $\mathsf{CheckNoExtraVar}(\cdot)$ in Algorithm~\ref{alg:inc-cache-retrieval} line~\ref{alg-line:inc-cache-retrieval:formula-check}. 
Subsequently, from lines~\ref{alg-line:inc-cache-retrieval:pairwise-check-start}-\ref{alg-line:inc-cache-retrieval:pairwise-check-end}, we verify that each ADD that we retrieved is compatible with all other already retrieved ADD candidates in $\mathcal{B}$. Finally, for all constraints that are not represented by an ADD in $\mathcal{B}$, we insert an ADD representing each constraint into $\mathcal{B}$. Cache retrieval replaces lines~\ref{alg-line:pbcount-greedy-mc:preprocess} to~\ref{alg-line:pbcount-greedy-mc:compile-add-end} of Algorithm~\ref{alg:pbcount-greedy-mc}. 
It is worth noting that caching ADDs requires us to disable preprocessing currently, as there is a need to maintain a unique id for each constraint and also a fixed variable-to-constraint relation. The restriction arises from the fact that we have to maintain ADD to variable mapping for ADDs in the cache to perform retrieval compatibility checks in Algorithm~\ref{alg:inc-cache-retrieval}. Preprocessing might remove variables and modify constraints, thus invalidating cached ADDs. As such, preprocessing is disabled when handling incremental counting.

\section{Experiments} \label{sec:experiments}

\begin{figure*}[htb]
    \centering
    \begin{subfigure}{0.33\textwidth}
        \includegraphics[width=\linewidth]{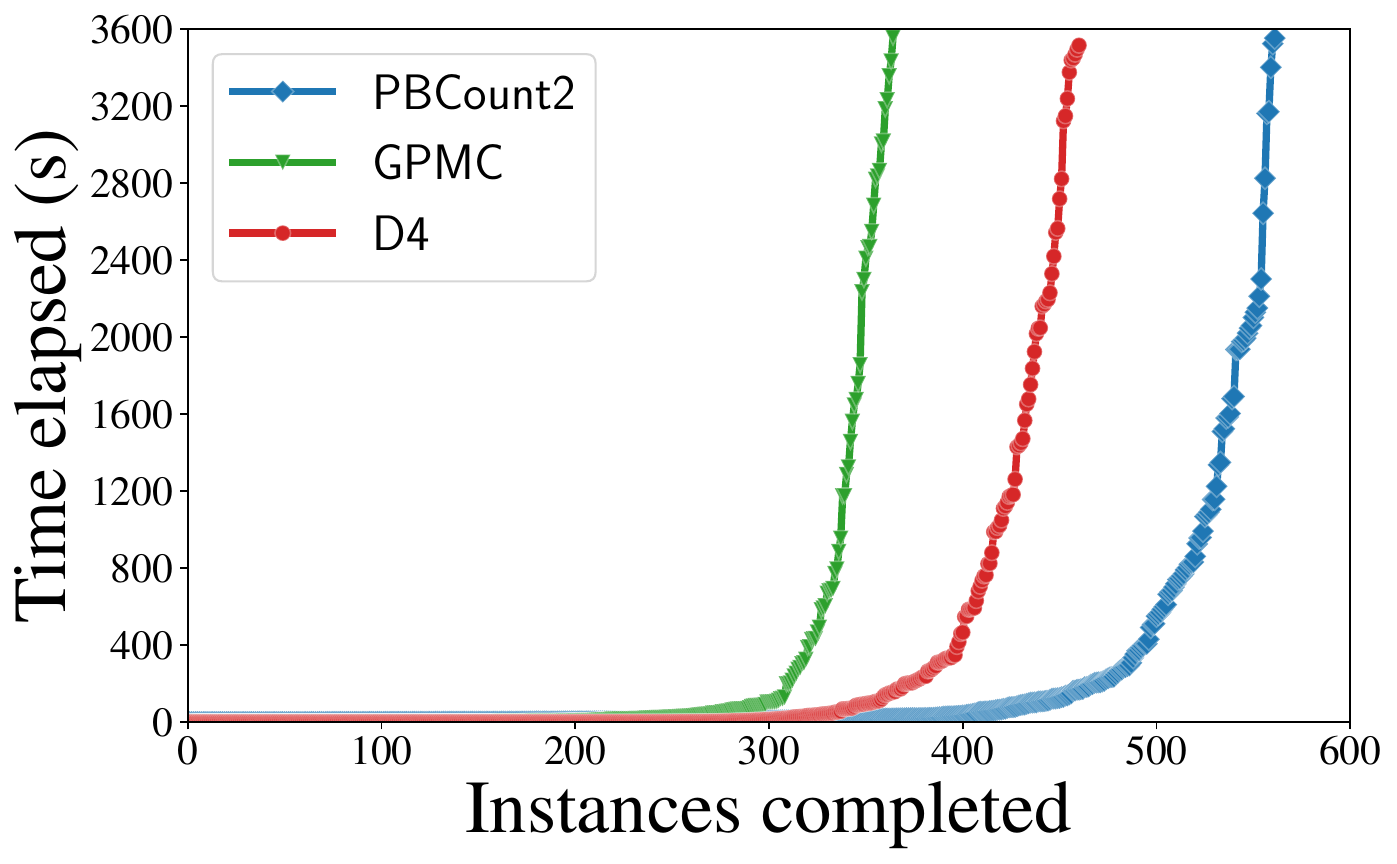}
        \caption{Auction}
    \end{subfigure}%
    \hfill
    \begin{subfigure}{0.33\textwidth}
        \includegraphics[width=\linewidth]{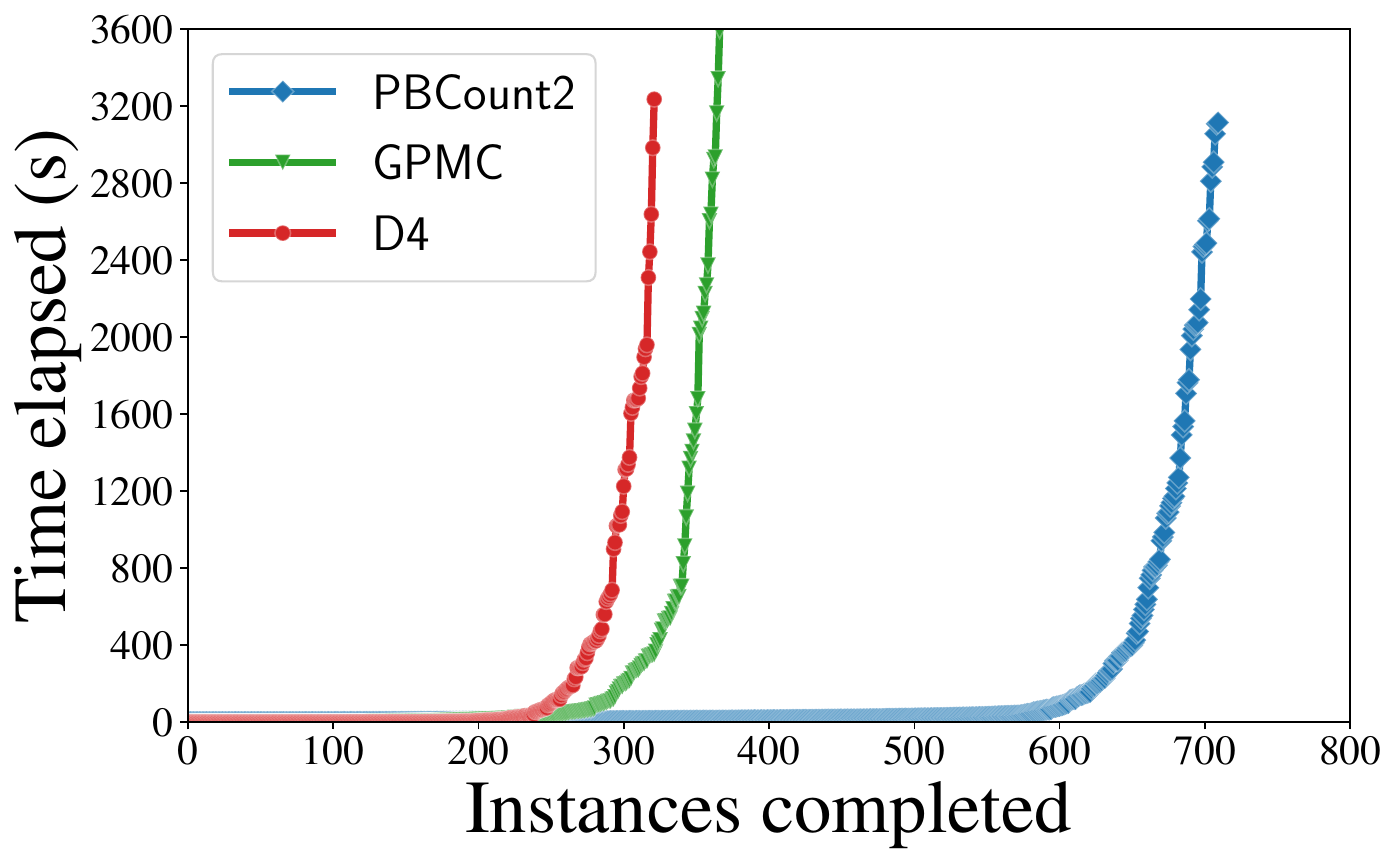}
        \caption{$\mathcal{M}$-dim Knapsack}
    \end{subfigure}%
    \hfill
    \begin{subfigure}{0.33\textwidth}
        \includegraphics[width=\linewidth]{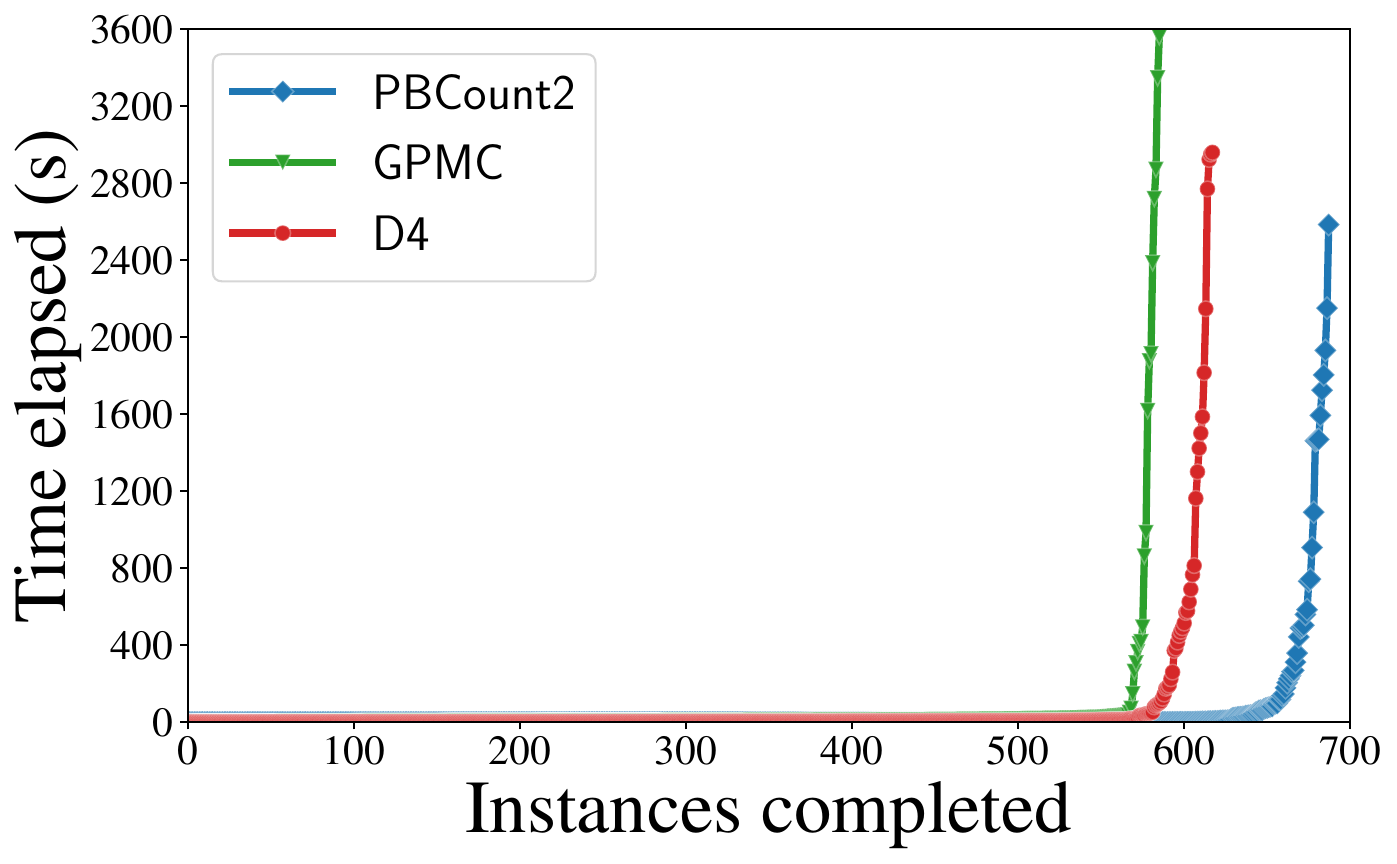}
        \caption{Sensor placement}
    \end{subfigure}
    \caption{Runtime cactus plots of {\pbcountcg} and competing methods for each benchmark set, for projected model counting.}
    \label{fig:individual-cactus-plot-projmc}
\end{figure*}

In this section, we detail the extensive evaluations conducted to investigate the performance of {\pbcountcg}'s new features, namely projected PB model counting and incremental PB model counting. Specifically, we investigate the following:

\begin{description}
    \item[RQ 1] How does {\pbcountcg} perform on projected PB model counting?
    \item[RQ 2] How does {\pbcounttwo} perform under incremental settings in comparison to the state-of-the-art {\pbcount} and CNF counters?
\end{description}

\subsection{Experiment Setup}
In our experiments, we derived synthetic benchmarks from prior work~\cite{YM24}, which has 3 benchmark sets - \textit{auction}, \textit{multi-dimension knapsack}, and \textit{sensor placement}. We defer the benchmark statistics to the appendix. The benchmarks are as follows:
\begin{description}
    \item[RQ 1 Benchmarks: ] Projected benchmarks are adapted from prior work~\cite{YM24}, with the projection set being randomly selected from the original variables. The CNF counters use benchmarks converted using {\pblib}~\cite{PS15}, with the same projection set.
    \item[RQ 2 Benchmarks: ] Incremental benchmarks are also adapted from prior work~\cite{YM24}. Each benchmark involves computing the model count for a given PB formula (step 1) and each of the 2 (4) subsequent modification steps for 3-step (5-step) configurations. Each modification step involves modifying an existing constraint. The modifications for the auction benchmarks correspond to updates to utility values, i.e. when preferences change. The modifications for the knapsack benchmarks correspond to changes to constraints of a particular dimension. Finally, the modifications to the sensor placement benchmarks correspond to additional redundancy requirements for sensors at important locations in the graph. We provided each step of the incremental benchmarks as separate instances to competing approaches, as none supported incremental counting.
\end{description}

In the experiments, we ran each benchmark instance using 1 core of an AMD EPYC 7713 processor, 16GB memory, and a timeout of 3600 seconds. We implemented {\pbcountcg} in C++, using CUDD library~\cite{S15} with double precision and the same ADD variable ordering (\textit{MCS}) as {\pbcount} and {\addmc}~\cite{DPV20a,YM24}. We compared {\pbcountcg} against state-of-the-art projected CNF model counters {\dfour} and {\gpmc}\footnote{Winners of PMC track of Model Counting Competition 2023} for \textbf{RQ1}, with the help of PB to CNF conversion tool {\pblib}. It is worth noting that the converted CNF instances have to be projected onto the original set of PB variables, as auxilary variables are introduced in the conversion process. For incremental benchmarks of \textbf{RQ2}, we compared {\pbcountcg} against {\pbcount} as well as CNF counters {\dfour} and {\gpmc}.

\subsection{RQ 1: Projected PB Model Counting Performance}

\begin{table}[h!]
    \centering
    \begin{small}
    \begin{NiceTabular}{l|r|r|r}
    \toprule
    Benchmarks                      & {\dfour}      & {\gpmc}       & {\pbcountcg}        \\
    \midrule
    Auction                         & 460           & 364           & \textbf{561}        \\
    $\mathcal{M}$-dim knapsack      & 321           & 366           & \textbf{709}        \\
    Sensor placement                & 617           & 585           & \textbf{687}        \\
    \midrule
    Total                           & 1398          & 1315          & \textbf{1957}       \\
    \bottomrule
    \end{NiceTabular}
    \end{small}
    \caption{Number of projected benchmark instances completed by {\dfour}, {\gpmc} and {\pbcountcg} in 3600s. A higher number is better.}
    \label{tab:benchmark-projected-heuristic}
\end{table}

\begin{figure*}[htb]
    \centering
    \begin{subfigure}{0.33\textwidth}
        \includegraphics[width=\linewidth]{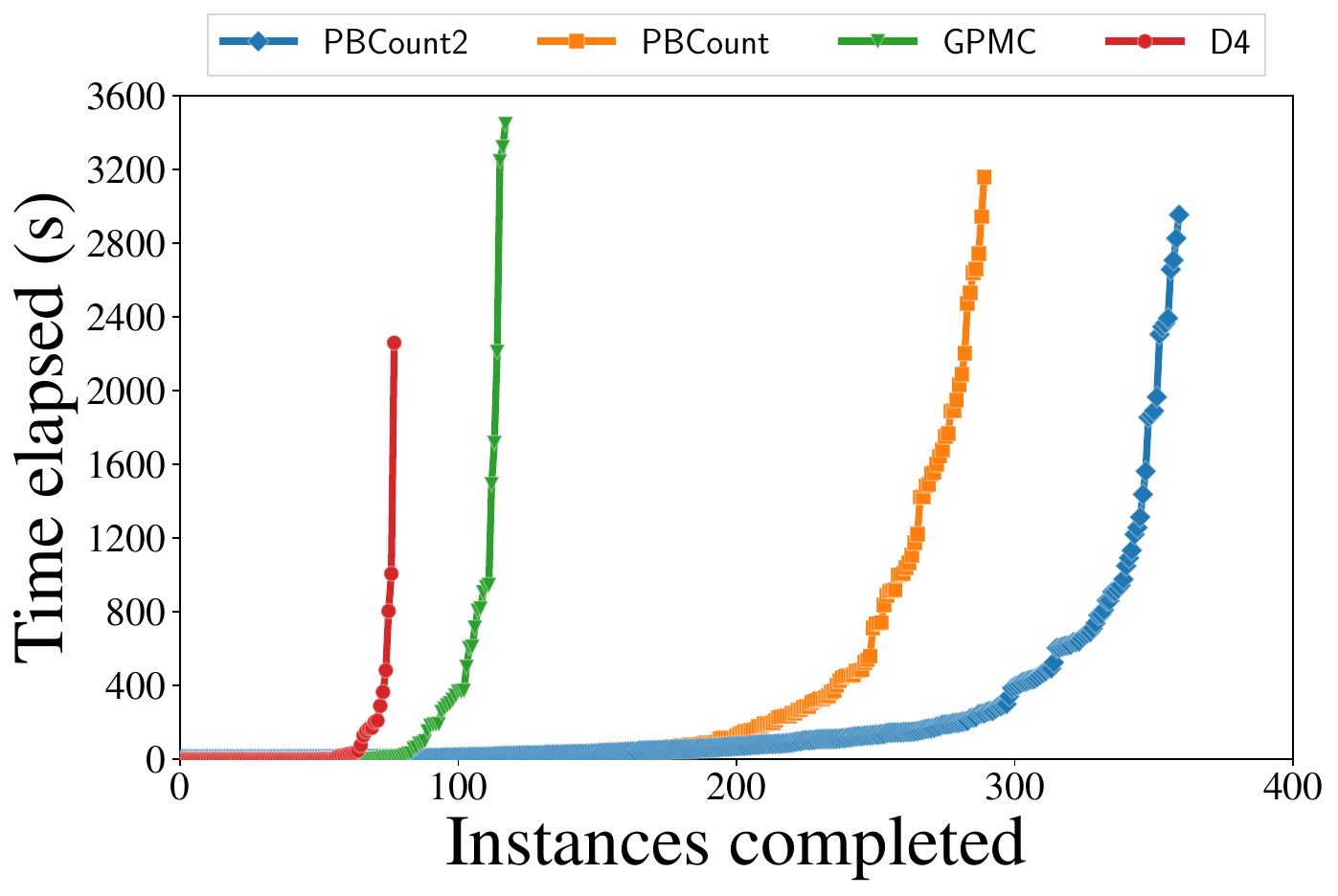}
        \caption{Auction}
    \end{subfigure}%
    \hfill
    \begin{subfigure}{0.33\textwidth}
        \includegraphics[width=\linewidth]{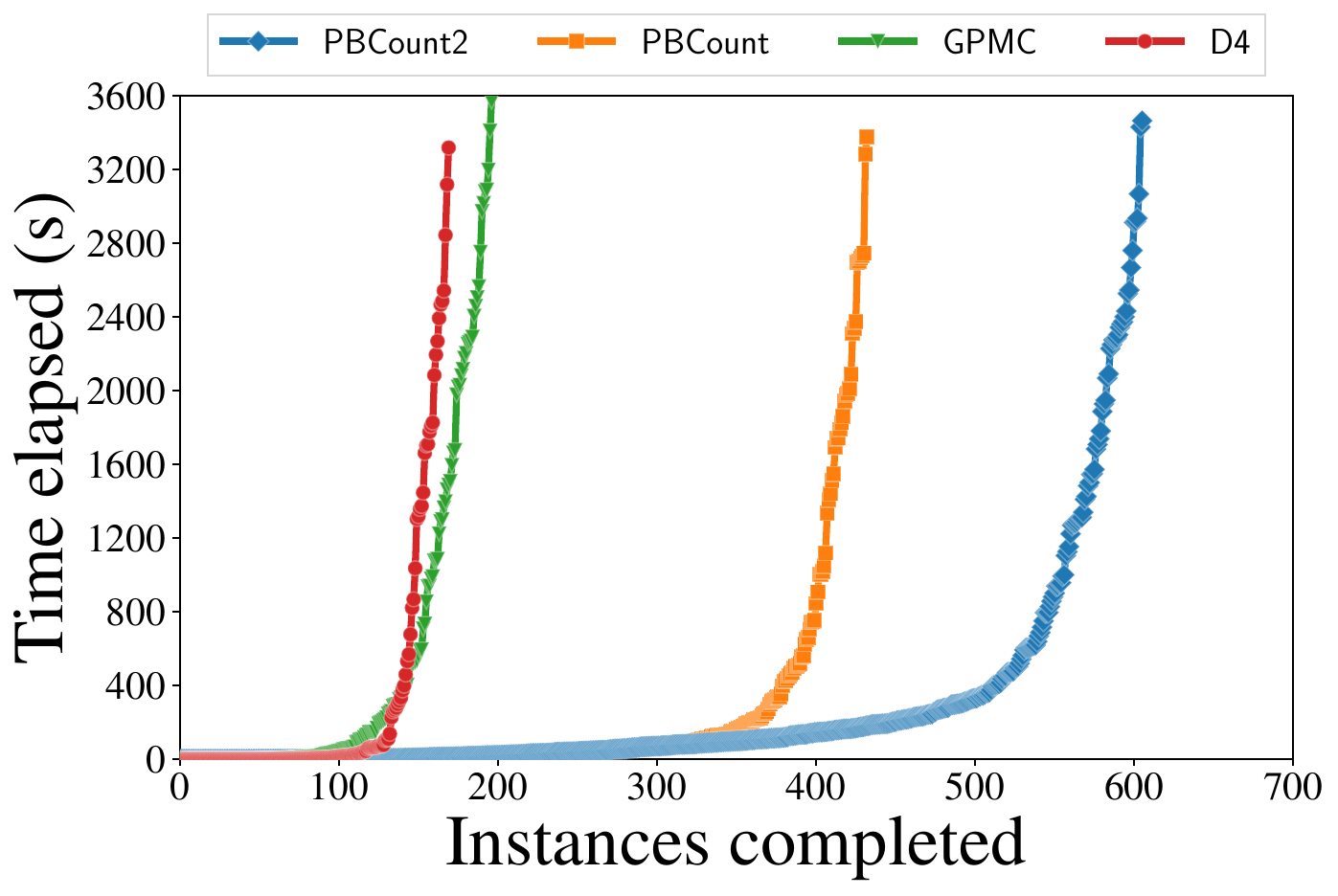}
        \caption{$\mathcal{M}$-dim Knapsack}
    \end{subfigure}%
    \hfill
    \begin{subfigure}{0.33\textwidth}
        \includegraphics[width=\linewidth]{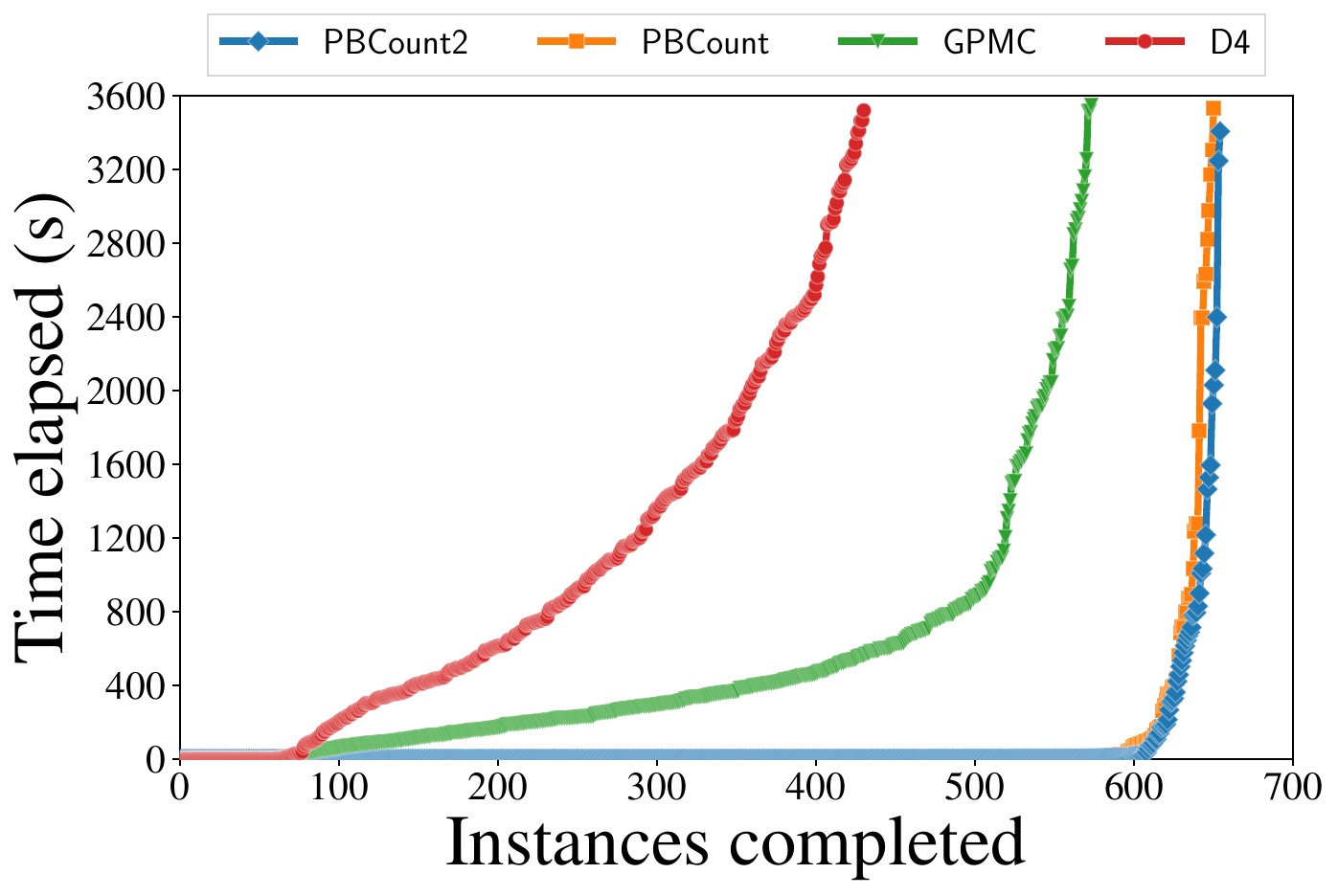}
        \caption{Sensor placement}
    \end{subfigure}
    \caption{Runtime cactus plots of {\pbcountcg} and competing methods for each benchmark set, for incremental counting with 5 steps.}
    \label{fig:5step-incremental-cactus}
\end{figure*}

\begin{table*}[h!]
    \centering
    \begin{small}
    \begin{NiceTabular}{llr|r|r|r}
    \toprule
    Experiment          & Counter                   & Auction           & $\mathcal{M}$-dim knapsack    & Sensor placement  & Total \\
    \midrule
    \Block{4-1}{3-step} & {\dfour}                  & 85                & 179                           & 536               & 800  \\
                        & {\gpmc}                   & 130               & 216                           & 586               & 932  \\
                        & {\pbcount}                & 312               & 458                           & 657               & 1427 \\
                        & {\pbcounttwo}             & \textbf{369}      & \textbf{629}                  & \textbf{660}      & \textbf{1658} \\
    \midrule
    \Block{4-1}{5-step} & {\dfour}                  & 77                & 169                           & 430               & 676  \\
                        & {\gpmc}                   & 117               & 196                           & 573               & 886  \\
                        & {\pbcount}                & 289               & 432                           & 650               & 1371 \\
                        & {\pbcounttwo}             & \textbf{359}      & \textbf{605}                  & \textbf{654}      & \textbf{1618} \\
    \bottomrule
    \CodeAfter
        \tikz \draw [very thick] (1-|3) |- (10-|3) ; 
    \end{NiceTabular}
    \end{small}
    \caption{Number of incremental benchmark instances completed by {\pbcounttwo} and competing methods in 3600s. `3-step' indicates results of incremental {\pbcounttwo} with 3 counting steps, and `5-step' indicates that with 5 counting steps. A higher number is better.}
    \label{tab:benchmark-incremental-steps-combined}
\end{table*}

We conducted extensive evaluations to understand the performance of {\pbcountcg} compared to state-of-the-art projected CNF model counters {\dfour} and {\gpmc}~\cite{LM17,SHS17}. 
We show the results in Table~\ref{tab:benchmark-projected-heuristic} and cactus plots in Figure~\ref{fig:individual-cactus-plot-projmc}. {\pbcountcg} is able to complete 1957 instances, demonstrating a substantial lead over the 1398 instances of {\dfour} and the 1315 instances of {\gpmc}. Overall, {\pbcountcg} solves around 1.40$\times$ the number of instances of {\dfour} and 1.49$\times$ that of {\gpmc}, highlighting the efficacy of {\pbcounttwo} in projected PB model counting tasks.

\subsection{RQ 2: Incremental Counting Performance}
We conducted experiments to analyze the performance of {\pbcountcg}'s incremental mode against the state-of-the-art {\pbcount} and CNF counters. In our experiments, we looked at the 3-step and 5-step benchmark configurations. The experiments were run with a total timeout of 3600s for the two benchmark configurations. The results are shown in Table~\ref{tab:benchmark-incremental-steps-combined}.

In the 3-step benchmarks, {\pbcounttwo} shows a considerable performance advantage over {\dfour}, {\gpmc}, and {\pbcount}. More specifically, {\pbcounttwo} completed 2.07$\times$ the number of benchmarks completed by {\dfour}, 1.78$\times$ that of {\gpmc}, and 1.16$\times$ that of {\pbcount} respectively. Across all 3 incremental benchmark sets for 3-step experiments, {\pbcounttwo} performs the same as or better than {\pbcount}. As we move to 5-step benchmarks, we see {\pbcounttwo} having a more significant lead over the competing methods. Overall, {\pbcounttwo} completes 2.39$\times$ the number of instances of {\dfour}, 1.83$\times$ that of {\gpmc}, and 1.18$\times$ that of the original {\pbcount}. Notably when moving from 3-step benchmarks to 5-step benchmarks, the drop in the number of completed incremental benchmarks of {\pbcounttwo} is only 40, compared to 56 for {\pbcount}, 46 for {\gpmc}, and 124 for {\dfour}. Additionally, we show the cactus plots of each set of incremental benchmarks under 5-step benchmark configuration in Figure~\ref{fig:5step-incremental-cactus}. The plot further highlights the runtime advantages of {\pbcounttwo} over competing approaches. Overall, the evaluations demonstrate the efficacy of {\pbcounttwo} at incremental PB model counting, as opposed to existing approaches that can only treat each incremental step as a separate model counting instance.

\subsection{Discussion}

Through the extensive evaluations, we highlighted {\pbcounttwo}'s superior performance in both projected PB model counting as well as incremental PB model counting benchmarks. To the best of our knowledge, {\pbcounttwo} is the first counter to support exact projected PB model counting, and also the first counter to support incremental counting. Despite the practical trade-offs required to support incremental counting, such as the lack of preprocessing and overhead of searching the cache, {\pbcounttwo} still demonstrated promising runtime advantages over competing methods. Additionally, {\pbcounttwo} also demonstrated the performance benefits of natively handling PB formulas for projected PB model counting, rather than using the convert-and-count approach with state-of-the-art projected CNF counters.

\section{Conclusion} \label{sec:conclusion}

In this work, we introduced {\pbcountcg}, the first exact PB model counter that supports a) projected model counting and b) incremental model counting, using our {\greedymergeshort} heuristic and caching mechanism. Our evaluations highlight {\pbcountcg}'s superior runtime performance in both projected model counting and incremental counting settings, completing 1.40$\times$ and 1.18$\times$ the number of instances of the best competing approaches respectively. While {\pbcountcg} supports incremental counting, it requires preprocessing techniques to be disabled to maintain coherent cached ADD to constraint and variable mappings. It would be of interest to overcome the preprocessing restriction in future works, perhaps by introducing the preprocessing techniques as inprocessing steps or by storing additional metadata about the preprocessing process. We hope the newly introduced capabilities of projection and incrementality lead to wider adoption and increased interest in PB model counting in general.

\section*{Acknowledgments}
The authors thank the reviewers for providing feedback. This work was supported in part by the Grab-NUS AI Lab, a joint collaboration between GrabTaxi Holdings Pte. Ltd. and National University of Singapore, and the Industrial Postgraduate Program (Grant: S18-1198-IPP-II), funded by the Economic Development Board of Singapore. This work was supported in part by National Research Foundation Singapore under its NRF Fellowship Programme [NRF-NRFFAI1-2019-0004], Ministry of Education Singapore Tier 2 grant [MOE-T2EP20121-0011], Ministry of Education Singapore Tier 1 grant [R-252-000-B59-114], and Natural Sciences and Engineering Research Council of Canada (NSERC) [RGPIN-2024-05956]. The computational work for this article was performed on resources of the National Supercomputing Centre, Singapore.

\bibliography{reference.bib}

\begin{thebibliography}{34}
\providecommand{\natexlab}[1]{#1}

\bibitem[{Aziz et~al.(2015)Aziz, Chu, Muise, and Stuckey}]{ACMS15}
Aziz, R.~A.; Chu, G.; Muise, C.; and Stuckey, P.~J. 2015.
\newblock \#$\exists$SAT: Projected Model Counting.
\newblock In \emph{International Conference on Theory and Applications of
  Satisfiability Testing}.

\bibitem[{Bacchus, Dalmao, and Pitassi(2003)}]{BDP03}
Bacchus, F.; Dalmao, S.; and Pitassi, T. 2003.
\newblock Algorithms and complexity results for \#SAT and Bayesian inference.
\newblock \emph{44th Annual IEEE Symposium on Foundations of Computer Science,
  2003. Proceedings.}

\bibitem[{Bahar et~al.(1993)Bahar, Frohm, Gaona, Hachtel, Macii, Pardo, and
  Somenzi}]{BFGHMPF93}
Bahar, R.~I.; Frohm, E.~A.; Gaona, C.~M.; Hachtel, G.~D.; Macii, E.; Pardo, A.;
  and Somenzi, F. 1993.
\newblock Algebraic decision diagrams and their applications.
\newblock In \emph{International Conference on Computer Aided Design}.

\bibitem[{Blumrosen and Nisan(2007)}]{BN07}
Blumrosen, L.; and Nisan, N. 2007.
\newblock \emph{Algorithmic Game Theory}, chapter~11, 267--300.
\newblock Cambridge University Press.

\bibitem[{Bryant(1986)}]{B86}
Bryant, R.~E. 1986.
\newblock Graph-Based Algorithms for Boolean Function Manipulation.
\newblock \emph{IEEE Transactions on Computers}, C-35(8): 677--691.

\bibitem[{Dudek, Phan, and Vardi(2020{\natexlab{a}})}]{DPV20a}
Dudek, J.~M.; Phan, V. H.~N.; and Vardi, M.~Y. 2020{\natexlab{a}}.
\newblock ADDMC: Weighted Model Counting with Algebraic Decision Diagrams.
\newblock In \emph{AAAI Conference on Artificial Intelligence}.

\bibitem[{Dudek, Phan, and Vardi(2020{\natexlab{b}})}]{DPV20b}
Dudek, J.~M.; Phan, V. H.~N.; and Vardi, M.~Y. 2020{\natexlab{b}}.
\newblock DPMC: Weighted Model Counting by Dynamic Programming on Project-Join
  Trees.
\newblock In \emph{International Conference on Principles and Practice of
  Constraint Programming}.

\bibitem[{Dudek, Phan, and Vardi(2021)}]{DPM21}
Dudek, J.~M.; Phan, V. H.~N.; and Vardi, M.~Y. 2021.
\newblock ProCount: Weighted Projected Model Counting with Graded Project-Join
  Trees.
\newblock In \emph{International Conference on Theory and Applications of
  Satisfiability Testing}.

\bibitem[{Due{\~n}as-Osorio et~al.(2017)Due{\~n}as-Osorio, Meel, Paredes, and
  Vardi}]{DMPV17}
Due{\~n}as-Osorio, L.; Meel, K.~S.; Paredes, R.; and Vardi, M.~Y. 2017.
\newblock Counting-Based Reliability Estimation for Power-Transmission Grids.
\newblock In \emph{AAAI Conference on Artificial Intelligence}.

\bibitem[{Fan, Miller, and Mitra(2020)}]{FMM20}
Fan, C.; Miller, K.; and Mitra, S. 2020.
\newblock Fast and Guaranteed Safe Controller Synthesis for Nonlinear Vehicle
  Models.
\newblock \emph{Computer Aided Verification}, 12224: 629 -- 652.

\bibitem[{Fazekas, Biere, and Scholl(2019)}]{FBS19}
Fazekas, K.; Biere, A.; and Scholl, C. 2019.
\newblock Incremental Inprocessing in SAT Solving.
\newblock In \emph{International Conference on Theory and Applications of
  Satisfiability Testing}.

\bibitem[{Gens and Levner(1980)}]{GL80}
Gens, G.; and Levner, E. 1980.
\newblock Complexity of approximation algorithms for combinatorial problems: a
  survey.
\newblock \emph{SIGACT News}, 12: 52--65.

\bibitem[{Jackson(2019)}]{J19}
Jackson, D. 2019.
\newblock Alloy: a language and tool for exploring software designs.
\newblock \emph{Commun. ACM}, 62.

\bibitem[{Klebanov, Manthey, and Muise(2013)}]{KMC13}
Klebanov, V.; Manthey, N.; and Muise, C. 2013.
\newblock SAT-Based Analysis and Quantification of Information Flow in
  Programs.
\newblock In \emph{International Conference on Quantitative Evaluation of
  Systems}.

\bibitem[{Korhonen and J{\"a}rvisalo(2021)}]{KJ21}
Korhonen, T.; and J{\"a}rvisalo, M. 2021.
\newblock Integrating Tree Decompositions into Decision Heuristics of
  Propositional Model Counters.
\newblock In \emph{International Conference on Principles and Practice of
  Constraint Programming}.

\bibitem[{Lagniez and Marquis(2017)}]{LM17}
Lagniez, J.-M.; and Marquis, P. 2017.
\newblock An Improved Decision-DNNF Compiler.
\newblock In \emph{International Joint Conference on Artificial Intelligence}.

\bibitem[{Lai, Meel, and Yap(2021)}]{LMY21}
Lai, Y.; Meel, K.~S.; and Yap, R. H.~C. 2021.
\newblock The Power of Literal Equivalence in Model Counting.
\newblock In \emph{AAAI Conference on Artificial Intelligence}.

\bibitem[{Latour, Sen, and Meel(2023)}]{LSM23}
Latour, A. L.~D.; Sen, A.; and Meel, K.~S. 2023.
\newblock Solving the Identifying Code Set Problem with Grouped Independent
  Support.
\newblock In \emph{Proceedings of the 32nd International Joint Conference on
  Artificial Intelligence}.

\bibitem[{{Le Berre} et~al.(2018){Le Berre}, Marquis, Mengel, and
  Wallon}]{LMMW18}
{Le Berre}, D.; Marquis, P.; Mengel, S.; and Wallon, R. 2018.
\newblock Pseudo-Boolean Constraints from a Knowledge Representation
  Perspective.
\newblock In \emph{International Joint Conference on Artificial Intelligence}.

\bibitem[{Lotz et~al.(2023)Lotz, Goel, Dutertre, Kiesl-Reiter, Kong, Majumdar,
  and Nowotka}]{LGDK23}
Lotz, K.; Goel, A.; Dutertre, B.; Kiesl-Reiter, B.; Kong, S.; Majumdar, R.; and
  Nowotka, D. 2023.
\newblock Solving String Constraints Using SAT.
\newblock In \emph{International Conference on Computer Aided Verification}.

\bibitem[{Nadel(2022)}]{N22}
Nadel, A. 2022.
\newblock Introducing Intel(R) SAT Solver.
\newblock In \emph{International Conference on Theory and Applications of
  Satisfiability Testing}.

\bibitem[{Nadel and Ryvchin(2012)}]{NR12}
Nadel, A.; and Ryvchin, V. 2012.
\newblock Efficient SAT Solving under Assumptions.
\newblock In \emph{International Conference on Theory and Applications of
  Satisfiability Testing}.

\bibitem[{Nadel, Ryvchin, and Strichman(2014)}]{NRS14}
Nadel, A.; Ryvchin, V.; and Strichman, O. 2014.
\newblock Ultimately Incremental SAT.
\newblock In \emph{International Conference on Theory and Applications of
  Satisfiability Testing}.

\bibitem[{Narodytska et~al.(2019)Narodytska, Shrotri, Meel, Ignatiev, and
  Marques-Silva}]{NSMIM19}
Narodytska, N.; Shrotri, A.~A.; Meel, K.~S.; Ignatiev, A.; and Marques-Silva,
  J. 2019.
\newblock Assessing Heuristic Machine Learning Explanations with Model
  Counting.
\newblock In \emph{International Conference on Theory and Applications of
  Satisfiability Testing}.

\bibitem[{Philipp and Steinke(2015)}]{PS15}
Philipp, T.; and Steinke, P. 2015.
\newblock PBLib - A Library for Encoding Pseudo-Boolean Constraints into CNF.
\newblock In \emph{International Conference on Theory and Applications of
  Satisfiability Testing}.

\bibitem[{Pisinger(2005)}]{P05}
Pisinger, D. 2005.
\newblock Where are the hard knapsack problems?
\newblock \emph{Computers \& Operations Research}, 32: 2271--2284.

\bibitem[{Ryosuke~Suzuki and Sakai(2017)}]{SHS17}
Ryosuke~Suzuki, K.~H.; and Sakai, M. 2017.
\newblock Improvement of Projected Model-Counting Solver with Component
  Decomposition Using SAT Solving in Components.
\newblock In \emph{JSAI Technical Report}.

\bibitem[{Sharma et~al.(2019)Sharma, Roy, Soos, and Meel}]{SRSM19}
Sharma, S.; Roy, S.; Soos, M.; and Meel, K.~S. 2019.
\newblock GANAK: A Scalable Probabilistic Exact Model Counter.
\newblock In \emph{Proceedings of International Joint Conference on Artificial
  Intelligence}.

\bibitem[{Sinz(2005)}]{S05}
Sinz, C. 2005.
\newblock Towards an Optimal CNF Encoding of Boolean Cardinality Constraints.
\newblock In \emph{International Conference on Principles and Practice of
  Constraint Programming}.

\bibitem[{Somenzi(2015)}]{S15}
Somenzi, F. 2015.
\newblock CUDD: CU decision diagram package - release 3.0.0.

\bibitem[{Yamada et~al.(2015)Yamada, Kitamura, Artho, Choi, Oiwa, and
  Biere}]{YKACOB15}
Yamada, A.; Kitamura, T.; Artho, C.; Choi, E.-H.; Oiwa, Y.; and Biere, A. 2015.
\newblock Optimization of Combinatorial Testing by Incremental SAT Solving.
\newblock \emph{2015 IEEE 8th International Conference on Software Testing,
  Verification and Validation (ICST)}, 1--10.

\bibitem[{Yang and Meel(2021)}]{YM21}
Yang, J.; and Meel, K.~S. 2021.
\newblock Engineering an Efficient PB-XOR Solver.
\newblock In \emph{International Conference on Principles and Practice of
  Constraint Programming}.

\bibitem[{Yang and Meel(2024)}]{YM24}
Yang, S.; and Meel, K.~S. 2024.
\newblock Engineering an Exact Pseudo-Boolean Model Counter.
\newblock In \emph{Proceedings of the 38th Annual AAAI Conference on Artificial
  Intelligence}.

\bibitem[{Yu et~al.(2017)Yu, Zhang, Liu, Ciesielski, and Holcomb}]{YZLCH17}
Yu, C.; Zhang, X.; Liu, D.; Ciesielski, M.~J.; and Holcomb, D.~E. 2017.
\newblock Incremental SAT-Based Reverse Engineering of Camouflaged Logic
  Circuits.
\newblock \emph{IEEE Transactions on Computer-Aided Design of Integrated
  Circuits and Systems}, 36: 1647--1659.

\end{thebibliography}

\clearpage
\onecolumn
\appendix
\clearpage
\onecolumn

\section{Benchmark Statistics}
We detail the statistics of the different benchmark sets in Table~\ref{app:tab:benchmark-stats}. As shown in the table, sensor placement settings have much more constraints (number of constraints) than auction and multi-dimension knapsack settings.

\begin{table}[h!]
    \centering
    \begin{small}
    \begin{NiceTabular}{l|l|r|r|r|r|r}
    \toprule
    \Block{1-2}{Benchmark stats}                &                   & $25\%$ & $50\%$ & $75\%$ & Avg. & Std. \\
    \midrule
    \Block{2-1}{Auction}                       & \# var            & 60.00  & 91.00  & 131.00 & 98.28  & 51.83 \\
                                                & \# constraint     & 4.00   & 9.00   & 14.00  & 9.32   & 5.58 \\
    \midrule
    \Block{2-1}{$\mathcal{M}$-dim knapsack}     & \# var            & 86.75  & 164.00 & 234.00 & 159.02 & 84.30 \\
                                                & \# constraint     & 6.00   & 10.00  & 15.00  & 10.41  & 5.68 \\
    \midrule
    \Block{2-1}{Sensor placement}               & \# var            & 25.00  & 81.50   & 190.00   & 111.63 & 91.54 \\
                                                & \# constraint     & 291.00 & 1210.00 & 7056.50 & 5561.22 & 8661.58 \\
    \bottomrule
    \end{NiceTabular}
    \end{small}
    \caption{Statistics of the benchmark sets. `25\%' refers to 25th percentile, `50\%' refers to median and `75\%' refers to 75th percentile. `Avg' indicates the average value and `Std.' indicates standard deviation. `\# var' indicates number of variables and `\# constraint' indicates number of constraints.}
    \label{app:tab:benchmark-stats}
\end{table}

\section{Additional Experiments}
We performed additional experiments to understand the impact of preprocessing on incremental model counting, as we had to disable preprocessing in {\pbcounttwo} when handling incremental counting setting. In particular, we compared {\pbcount} with preprocessing switched on and off for incremental model counting in Table~\ref{app:tab:incremental-nopp-pbcount}.

\begin{table*}[h!]
    \centering
    \begin{small}
    \begin{NiceTabular}{llr|r|r|r}
    \toprule
    Experiment          & Counter                   & Auction           & $\mathcal{M}$-dim knapsack    & Sensor placement  & Total \\
    \midrule
    \Block{2-1}{3-step} & {\pbcount}                & 312               & 458                           & 657               & 1427  \\
                        & {\pbcount}-nopp           & 309               & 440                           & 652               & 1401  \\
    \midrule
    \Block{2-1}{5-step} & {\pbcount}                & 289               & 432                           & 650               & 1371  \\
                        & {\pbcount}-nopp           & 289               & 414                           & 644               & 1347  \\
    \bottomrule
    \CodeAfter
        \tikz \draw [very thick] (1-|3) |- (6-|3) ; 
    \end{NiceTabular}
    \end{small}
    \caption{Number of incremental benchmark instances completed by {\pbcount} in 3600s, with and without preprocessing. `{\pbcount}-nopp' indicates when preprocessing is disabled. `3-step' indicates results of {\pbcount} with 3 counting steps, and `5-step' indicates that with 5 counting steps. Each step is treated as a separate instance by {\pbcount}.}
    \label{app:tab:incremental-nopp-pbcount}
\end{table*}

\section{Algorithm Correctness}

For completeness, we provide the necessary notations, definitions, and proofs adapted from prior work~\cite{DPM21} to show the correctness of our projected model counting algorithm, Algorithm~\ref{alg:pbcount-greedy-mc} in the main paper.

\subsection{Definitions}

\begin{definition}[\textbf{Projected Valuation}] \label{app:def:proj-valuation}
    Let $(F,X,Y)$ be a PB model counting instance, with projection set $X$, non-projection set $Y$, and PB formula $F$. Let $\mathcal{T}$ be an $X,Y$ graded project join tree of $F$. The projected-valuation of each node $n \in \mathcal{V}(T)$, denoted $g_{n}$, is defined as:
    \begin{equation*}
        g_{n} = 
        \begin{cases}
            [\gamma(n)] & \text{if $n \in \mathcal{L}(T)$} \\
            \sum_{\pi(n)} ( \prod_{o \in \child{n}} g_{o} ) & \text{if $n \in \mathcal{I}_X$}\\
            \exists_{\pi(n)} (\prod_{o \in \child{n}} g^{W}_{o}) & \text{if $n \in \mathcal{I}_Y$}
        \end{cases}
    \end{equation*}
    Where $[\gamma(n)]$ is the representation of PB constraint $\gamma(n) \in F$ and \child{n} refers to the set of immediate child nodes of $n$ in the project join tree.
\end{definition}

\begin{definition}[\textbf{Early Projection}] \label{app:def:early-project}
    Let X and Y be sets of variables. For all functions $f$ defined over $X$ and $g$ defined over $Y$, $f : 2^X \rightarrow \mathbb{R}$, and $g: 2^Y \rightarrow \mathbb{R}$, if $x \in X \setminus Y$ then $\sum_X (f \cdot g) = (\sum_X f) \cdot g$ and $\exists_X (f \cdot g) = (\exists_X f) \cdot g$
\end{definition}

\subsection{Additional Notations}
We introduce additional notations that are relevant to the proofs that follow. Let $\mathcal{T} = (T, r, \gamma, \pi)$ be a project join tree for PB formula $F$ and $n \in \mathcal{V}(T)$ be a node in $T$. $S(n) \in \mathcal{V}(T)$ denotes the set of all descendants of $n$, including itself. In other words, $S(n)$ is the set of vertices of the subtree rooted at $n$. Let $P(n) = \bigcup_{o \in S(n) \setminus \mathcal{L}(T)} \pi(o)$ be the set of variables projected in $S(n)$. Let $\Phi(n) = \{\gamma(l) : l \in \mathcal{L}(T) \cap S(n)\}$ be the set of PB constraints of $F$ that are mapped to the leaf nodes of $S(n)$.

\subsection{Proof}

\begin{lemma} \label{app:lemma1}
    Let $n$ be an internal node in a project join tree $\mathcal{T} = (T, r, \gamma, \pi)$. Let $o$ and $p$ be in \child{n}, and that $o \not= p$. Then $P(o) \cap \var{\Phi(q)} = \emptyset$.
\end{lemma}

\begin{proof}
    Let $x$ be a variable in $P(o)$, then $x \in \pi(s)$ for some internal node $s$ that is a descendant of $o$. Suppose that $x$ appears in some arbitrary PB constraint $c$ of PB formula $F$. By definition of project join tree (main paper), the leaf node $\gamma^{-1}(c)$ must be a descendant of $s$ and thus also a descendant of $o$. As $o$ and $q$ are sibling nodes, variable $x$ does not appear in any descendant leaf node of $q$ and $x \not \in \var{\Psi(q)}$. Thus $P(o) \cap \var{\Phi(q)} = \emptyset$.
\end{proof}

\begin{lemma} \label{app:lemma2}
    Let $(F,X,Y)$ be a weighted projected model counting instance and $\mathcal{T}$ be an $X,Y$ graded project join tree for $F$.
    For every node $n$ of $T$, let
    \begin{equation*}
        h_{n} =  \prod_{C \in \Phi(n)} [C] 
    \end{equation*}
    Then
    \begin{equation} \label{app:eq1}
        g_{n} = \sum_{P(n) \cap X} \bigexists_{P(n) \cap Y} h_{n}
    \end{equation}
\end{lemma}

\begin{proof}
    Let $n$ be an arbitrary node in $\mathcal{V}(T)$.\\

    \textbf{Case 1: } Suppose $n$ is a leaf node, that is $n \in \mathcal{L}(T)$.\\
    If $n$ is a leaf node, then $\Phi(n) = \{\gamma(n)\}$ and $P(n) = \emptyset$. By definition of $h_{n}$, $h_{n} =  \prod_{C \in \{\gamma(n)\}} [C] = [\gamma(n)]$. It follows that $\sum_{P(n) \cap X} \bigexists_{P(n) \cap Y} h_{n}$ becomes $\sum_{\emptyset} \bigexists_{\emptyset} h_{n} = [\gamma(n)]$. Notice that this is exactly $g_{n}$ by Definition~\ref{app:def:proj-valuation}, for the case where $n \in \mathcal{L}(T)$. \\

    \textbf{Case 2: } Now we look at the case where $n$ is an internal node of $T$, that is $n \in \mathcal{V}(T) \setminus \mathcal{L}(T)$. \\
    Suppose for each child node $o \in \child{n}$, $g_{o} = \sum_{P(o) \cap X} \bigexists_{P(o) \cap Y} h_{o}$. Then if we take the valuation product, and substituting Equation~\ref{app:eq1}, we have
    \begin{equation}
        \prod_{o \in \child{n}} g_{o} = \prod_{o \in \child{n}} \sum_{P(o) \cap X} \bigexists_{P(o) \cap Y} h_{o}
    \end{equation}

    Since for $o, q \in \child{n}$ and $o \not= q$ $P(o) \cap \var{\Phi(q)} = \emptyset$ by Lemma~\ref{app:lemma1}, then $P(o) \cap \var{h_{q}} = \emptyset$. As a result, by the applying early projection according to Definition~\ref{app:def:early-project}, we have
    \begin{equation} \label{app:eq3}
        \prod_{o \in \child{n}} g_{o} = \sum_{A \cap X} \prod_{o \in \child{n}} \bigexists_{P(o) \cap Y} h_{o} = \sum_{A \cap X} \bigexists_{A \cap Y} \prod_{o \in \child{n}}  h_{o}
    \end{equation}
    where $A$ is $\bigcup_{o \in \child{n}} P(o)$.

    Recall that the set of internal nodes, $\mathcal{V}(T) \setminus \mathcal{L}(T)$ , of an $\mathcal{T}$ is partitioned by $\mathcal{I}_X$ and $\mathcal{I}_Y$, definition of $X,Y$-graded project join tree (main paper). As such, there are two subcases, where $n$ is $\mathcal{I}_X$ or $\mathcal{I}_Y$.

    \textbf{Case 2a: } $n \in \mathcal{I}_Y$.
    For each $p \in S(n)$, $p \in \mathcal{I}_Y$ and $\pi{p} \subseteq Y$ by definition of $X,Y$-graded project join tree (main paper). As such $\bigcup_{o \in \child{n}} P(o)$ or $A$ is a subset of $Y$. Thus by Definition~\ref{app:def:proj-valuation} and Equation~\ref{app:eq3},
    \begin{equation*}
        g_n = \bigexists_{\pi(n)} \prod_{o \in \child{n}} g_o = \bigexists_{\pi(n)} \bigexists_A \prod_{o \in \child{n}} h_o = \bigexists_{P(n)} \prod_{o \in \child{n}} h_o 
    \end{equation*}
    
    Therefore,
    \begin{equation*}
        g_n = \bigexists_{P(n)} \prod_{o \in \child{n}} \prod_{C \in \Phi(o)} [C] = \bigexists_{P(n)} \prod_{C \in \Phi(n)} [C] = \bigexists_{P(n)} h_n
    \end{equation*}

    \textbf{Case 2b: } $n \in \mathcal{I}_X$.
    Notice that $\pi(n) \subseteq X$. Using Definition~\ref{app:def:proj-valuation} and Equation~\ref{app:eq3}, we have
    \begin{equation*}
        g_n = \sum_{\pi(n)} \prod_{o \in \child{n}} g_o = \sum_{\pi(n)} \left( \sum_{A \cap X} \bigexists_{A \cap Y} \prod_{o\in \child{n}} h_o \right)
    \end{equation*}
    Notice that $\pi(n) \cup A = P(n)$ and also $h_n = \prod_{o\in\child{n}} h_o$, therefore
    \begin{equation*}
        g_n = \sum_{P(n) \cap X} \bigexists_{P(n) \cap Y} \prod_{o\in\child{n}} h_o = \sum_{P(n) \cap X} \bigexists_{P(n) \cap Y} h_n
    \end{equation*}
\end{proof}

\begin{lemma} \label{app:lemma3}
    The computation process in Algorithm~\ref{alg:pbcount-greedy-mc} follows an $X,Y$ graded project join tree to compute projected valuation at each node.
\end{lemma} 

\begin{proof}
    We split the proof into two parts, first showing by construction that Algorithm~\ref{alg:pbcount-greedy-mc} follows an $X,Y$ graded project join tree. Subsequently, we show that the computation at each node is exactly according to the definition of projected valuation. 

    Let $T$ be a tree, with leaf nodes having a bijective relationship with the individual ADDs at the start of computation of Algorothm~\ref{alg:pbcount-greedy-mc} (line~\ref{alg-line:pbcount-greedy-mc:compile-add-end}). Let each merged intermediate ADD $\psi$ at lines~\ref{alg-line:pbcount-greedy-mc:cache1} and~\ref{alg-line:pbcount-greedy-mc:cache2} map to an internal node of $T$, with the internal node's descendants being all the ADDs (denoted $\varphi$) involved in the merging process at line~\ref{alg-line:pbcount-greedy-mc:nps-merge} and~\ref{alg-line:pbcount-greedy-mc:ps-merge} respectively. In addition, let the variable projected away in the merge and project iteration be the label of the respective internal node given by labelling function $\pi$. The root node of $T$ would correspond to the final merged ADD $\psi$ at line~\ref{alg-line:pbcount-greedy-mc:final-merge} of Algorithm~\ref{alg:pbcount-greedy-mc}. 

    Recall that an $X,Y$-graded project join tree $\mathcal{T}$ has two disjoint sets of internal nodes $\mathcal{I}_{X}$ and $\mathcal{I}_{Y}$, of $X$ grade and $Y$ graded respectively. The structural requirements of $X,Y$-graded project join tree is such that nodes in $\mathcal{I}_{Y}$ cannot have any node of $\mathcal{I}_{X}$ as descendant in $\mathcal{T}$. We let internal nodes of $T$ mapped in line~\ref{alg-line:pbcount-greedy-mc:nps-merge} of Algorithm~\ref{alg:pbcount-greedy-mc} to be of grade $I_Y$ and the internal nodes mapped in line~\ref{alg-line:pbcount-greedy-mc:ps-merge} to be of grade $I_X$. Notice that we labelled all internal nodes, hence $\{\mathcal{I}_X , \mathcal{I}_Y\}$ is a partition of all internal nodes, satisfying property 1 of definition of $X,Y$-graded project join tree (main paper). It is clear that properties 2 and 3 of $X,Y$-graded project join tree (main paper) are satisfied as the $\pi$ labels each internal node $n$ to the variables projected away from it.
    
    Since all internal nodes of grade $I_Y$ are produced before $I_X$ by design of Algorithm~\ref{alg:pbcount-greedy-mc}, none of the internal nodes of $\mathcal{I}_{X}$ would be descendants of internal nodes in $\mathcal{I}_{Y}$, satisfying property 4 of $X,Y$-graded project join tree (main paper). 
    
    Recall that a project join tree's labelling function $\pi$ should label each internal node with a set of variables such that the different labels of internal nodes partition the set of variables of the original PB formula $F$. For each variable $v$, notice that Algorithm~\ref{alg:pbcount-greedy-mc} only projects away $v$ once, after merging the corresponding ADDs containing $v$. Hence, there is no repetition of labels, and the labels of all internal nodes partitions $X \cup Y$.

    In addition, an $X,Y$-graded project join tree has leaf nodes that has one to one mapping $\gamma$ to constraints in the original formula. In $T$, recall that each leaf node is one-to-one mapped to an individual ADD at the start of Algorothm~\ref{alg:pbcount-greedy-mc}. Since each individual ADD directly represents a constraint in the original PB formula, leaf nodes of $T$ have a bijection to constraints in the PB formula by transitivity. It is clear that property 2 of project join tree (main paper) holds, because at each merge and project iteration, we merge all ADDs containing the selected variable (lines~\ref{alg-line:pbcount-greedy-mc:nps-merge} and~\ref{alg-line:pbcount-greedy-mc:ps-merge}).

    Since tree $T$ arising from the computation process of Algorithm~\ref{alg:pbcount-greedy-mc} meets the specifications of an $X,Y$-graded project join tree, the computation process of Algorithm~\ref{alg:pbcount-greedy-mc} indeed follows $T$, an $X,Y$-graded project join tree.

    Notice that at each leaf node $n \in \mathcal{L}(T)$, we have the representation of a constraint $[\gamma(n)]$, which is the individual ADD at the start of Algorithm~\ref{alg:pbcount-greedy-mc}. At each node $n \in \mathcal{I}_Y$, we merge the ADDs of descendants using the {\apply} operation with $\times$ operator and {\orproj} away variables labelled at $n$ on line~\ref{alg-line:pbcount-greedy-mc:cache1}, computing $\exists_{\pi(n)} (\prod_{o \in \child{n}} g_{o})$. Similarly at each node $n \in \mathcal{I}_X$, we again merge ADDs of descendants in the same manner but {\addproj} away variables on line~\ref{alg-line:pbcount-greedy-mc:cache2}, computing $\sum_{\pi(n)} (\prod_{o \in \child{n}} g_{o})$. Since Algorithm~\ref{alg:pbcount-greedy-mc} computations follow Definition~\ref{app:def:proj-valuation} at all nodes, it computes $g_n$.

    Thus, we showed that the computation process in Algorithm~\ref{alg:pbcount-greedy-mc} follows an $X,Y$ graded project join tree to compute projected valuation at each node.
\end{proof}

Recall that we stated Theorem~\ref{theorem:count-correctness} in the main paper, we restate it here for convenience and provide the proof as follows.

\begin{mythm}{\ref{theorem:count-correctness}}
    Let $F$ be a formula defined over $X \cup Y$ such that $X$ is the projection set, and $Y$ is the set of variables not in projection set, then given an instance $(F,X,Y)$, Algorithm ~\ref{alg:pbcount-greedy-mc} returns $c$ such that $c = \sum_{\beta \in 2^X} (\max_{\alpha \in 2^Y} [F](\alpha \cup \beta))$ 
\end{mythm}

\begin{proof}
    Let the root node of $\mathcal{T}$ be $r$. Notice that because $r$ is the root node, $\Phi(r) = F$ and $P(r) = X \cup Y$. 
    By Lemma~\ref{app:lemma3}, the computation process of Algorithm~\ref{alg:pbcount-greedy-mc} follows an $X,Y$ graded project join tree $\mathcal{T}$. 
    
    Using Lemma~\ref{app:lemma2},
    \begin{equation*}
        g_r = \sum_X \bigexists_Y h_r = \sum_X \bigexists_Y \prod_{C \in F} [C] = \sum_X \bigexists_Y [F] 
    \end{equation*}

    Notice that $\sum_X \bigexists_Y [F]$ is exactly $\sum_{\beta \in 2^X} (\max_{\alpha \in 2^Y} [F](\alpha \cup \beta))$. Hence, $g_r$ is the $Y$-projected count of $F$ and Algorithm~\ref{alg:pbcount-greedy-mc} produces the count $c = \sum_{\beta \in 2^X} (\max_{\alpha \in 2^Y} [F](\alpha \cup \beta))$.
\end{proof}

\end{document}